\documentclass{article}

    \usepackage[preprint]{neurips_2026}

\usepackage[utf8]{inputenc} % allow utf-8 input
\usepackage[T1]{fontenc}    % use 8-bit T1 fonts
\usepackage{hyperref}       % hyperlinks
\usepackage{url}            % simple URL typesetting
\usepackage{booktabs}       % professional-quality tables
\usepackage{amsfonts}       % blackboard math symbols
\usepackage{nicefrac}       % compact symbols for 1/2, etc.
\usepackage{microtype}      % microtypography
\usepackage{xcolor}         % colors

\usepackage{microtype}
\usepackage{graphicx}

\usepackage{textcomp}

\usepackage{amsmath,amssymb} % define this before the line numbering.
\usepackage{color, colortbl}

\usepackage{multirow}
\usepackage{bbm}

\usepackage{tikz}
\usepackage{comment}
\usepackage{amsthm}
\usepackage{algorithm}
\usepackage{algorithmic}

\usepackage[font=small]{caption}
\usepackage{subcaption}

\usepackage{cleveref}
\crefformat{section}{Sec.~#2#1#3}
\crefformat{subsection}{Sec.~#2#1#3}
\crefformat{subsubsection}{Sec.~#2#1#3}
\crefformat{figure}{Fig.~#2#1#3}
\crefformat{equation}{Eq.~#2#1#3}
\crefformat{table}{Tab.~#2#1#3}
\crefformat{algorithm}{Alg.~#2#1#3}

\usepackage{pifont}%

\definecolor{Gray}{gray}{0.9}

\usepackage{xspace}

\newtheorem{theorem}{Theorem}[section]
\newtheorem{proposition}[theorem]{Proposition}

\title{APO: Unsupervised Atomic Policy Optimization for 3D Structure Prediction of Atomic Systems}

\author{%
  Shentong Mo$^{1}$\thanks{Corresponding author: \texttt{shentongmo@gmail.com}.}, Yatao Bian$^2$ \\
  $^1$CMU, $^2$NUS \\
}

\begin{document}

\maketitle

\begin{abstract}

Predicting the 3D structures of atomic systems is fundamental to advancing material science and drug discovery. 
While flow-matching models (\textit{e.g.}, FlowDPO) have recently shown promise in this domain, their performance relies heavily on alignment with ground-truth coordinates via supervised preference learning. 
However, obtaining experimental labels for novel crystal phases or de novo proteins is prohibitively expensive, creating a bottleneck for structural modeling in data-scarce regimes.
In this work, we propose \textit{APO} (Atomic Policy Optimization), a fully unsupervised alignment framework that eliminates the need for ground-truth reference structures. APO adapts group-relative policy optimization to 3D atomic environments, utilizing a novel dual-reward mechanism: (i) a \textit{Spectral Consistency Score} that reinforces the policy's dominant latent structural modes through eigen-decomposition of sample similarities, and (ii) a \textit{Crystal Entropy Proxy} that enforces thermodynamic stability. Our framework enables the model to ``self-correct'' by identifying physically plausible configurations within sampled groups. 
Extensive benchmarks on crystal and antibody structure prediction demonstrate that APO consistently outperforms fully supervised baselines, achieving a new state-of-the-art in match rates and structural fidelity. Furthermore, we show that APO effectively straightens probability paths, significantly improving inference efficiency. Our results suggest that intrinsic physical consistency can serve as a superior guide for alignment compared to noisy, supervised coordinate matching.

\end{abstract}

\section{Introduction}

Predicting high-fidelity 3D structures of atomic systems is a fundamental challenge in computational biology and material science, with applications ranging from drug discovery to the development of novel functional materials~\cite{jumper2021highly, xie2021crystal}. Accurate 3D modeling is essential for understanding the physical properties of substances and predicting their behavior in complex environments. While traditional physics-based algorithms often focus on finding local energy optima via sampling methods like Molecular Dynamics (MD) or Markov Chain Monte Carlo (MCMC), recent advancements have shifted toward deep generative models that learn stable structure distributions directly from experimental data~\cite{jing2022torsional, luo2022antigen}.

Among these, flow-matching models~\cite{lipman2022flow, albergo2023building} have emerged as a powerful paradigm, offering more flexible and often straighter probability paths than traditional diffusion-based approaches~\cite{ho2020denoising}. A notable state-of-the-art framework in this domain is FlowDPO~\cite{jiao2024flowdpo}, which integrates Direct Preference Optimization (DPO)~\cite{rafailov2024direct} into flow-matching models to suppress sampling hallucinations and improve generation quality. FlowDPO operates by generating candidate structures, ranking them based on their distance to ground-truth coordinates, and using this preference data to align the model toward high-fidelity distributions.
However, a critical limitation of existing alignment methods like FlowDPO is their strict reliance on supervised signals. These frameworks require ground-truth reference structures, such as experimental crystal structures from the Materials Project~\cite{jain2013commentary} or validated protein folds from the PDB~\cite{berman2000protein}, to construct preference pairs and compute rewards like Root Mean Square Deviation (RMSD). In many scientific frontiers, such as the discovery of entirely new crystal phases or the design of \textit{de novo} proteins~\cite{watson2023de}, ground-truth data is prohibitively expensive or non-existent. This dependence on labeled supervision restricts the ability of current models to explore and align within the vast, unlabeled conformational space where physical principles, rather than labels, serve as the ultimate guide.

\begin{figure}[t]
\centering
% \fbox{\rule{0pt}{2in}
% \rule{0.8\linewidth}{0pt}}
\includegraphics[width=0.65\linewidth]{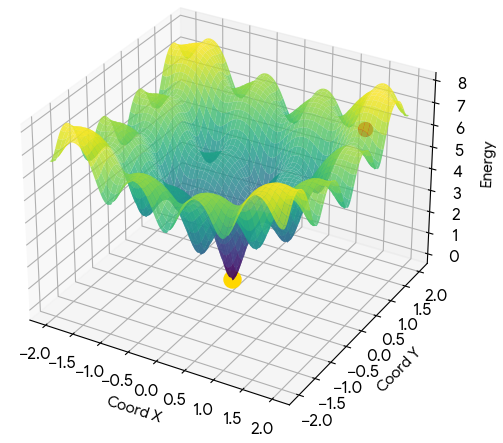}
\vspace{-0.5em}
\caption{3D representation of the atomic energy landscape, highlighting the challenge of local minima (hallucinations) versus the global minimum (stable structure) that the entropy reward targets.
}
\label{fig:energy_surface}
\vspace{-1.0em}
\end{figure}

The primary challenge in moving beyond supervision lies in deriving meaningful preference signals from the intrinsic geometry and physics of the system itself, without reverting to the computational bottleneck of expensive physics-based simulations. 
Unlike supervised alignment, where the ground-truth provides a clear directional gradient toward a global optimum, intrinsic rewards in atomic systems are often sparse and noisy. The potential energy surface of complex molecules and crystals is characterized by a vast number of local minima, as illustrated in Figure~\ref{fig:energy_surface}; consequently, simple physical proxies may inadvertently reward metastable or unphysical configurations that exhibit low energy but lack structural integrity.
Furthermore, defining a ``preference" in a multi-agent atomic environment requires capturing high-order geometric dependencies. While FlowDPO uses a point-to-point distance (RMSD) to a known target, an unsupervised agent must evaluate the quality of a generated structure by looking at the collective harmony of its atoms. This necessitates a mechanism that can distinguish between ``global" structural modes and ``local" noise without a reference template. Finally, there is a distributional shift challenge: as the policy evolves to minimize intrinsic energy or maximize consistency, it may collapse into a narrow set of overly simplistic structures (mode collapse), making it difficult to maintain the structural diversity required for applications like antibody design or polymorphic crystal discovery.

To address this, we introduce \textit{APO} (Atomic Policy Optimization), a fully unsupervised alignment framework that eliminates the need for ground-truth coordinates during training. APO adapts Group Relative Policy Optimization (GRPO,~\cite{shao2024deepseekmath}) to atomic environments, enabling iterative policy refinement driven purely by intrinsic feedback.
Our solution formulates a dual-reward mechanism to guide the alignment: (i) a \textit{spectral consistency score}, which utilizes the eigen-decomposition of a representation similarity matrix to reward candidates aligned with the policy's dominant latent structural modes; and (ii) a \textit{crystal entropy minimization} objective, which serves as a thermodynamic proxy for structural stability. By utilizing these signals within a group-based tournament, APO allows the model to ``self-correct" by preferring physically plausible and structurally consistent configurations over hallucinated ones.

We evaluate APO on two distinct scientific benchmarks: Crystal Structure Prediction and Antibody Structure Prediction. Our APO consistently outperforms the fully supervised FlowDPO across all metrics, achieving higher Match Rates in materials and lower RMSD in antibody CDR loop prediction. We demonstrate that APO successfully aligns various flow geometries, including Variance Preserving (VP) and Optimal Transport (OT) paths, with the most significant gains observed in OT paths due to a ``path-straightening" effect. Analysis reveals that APO-aligned models exhibit emergent periodic symmetries and thermodynamic stability that exceed models trained purely on coordinate-wise supervision. These results suggest that 3D atomic system alignment can be effectively driven by intrinsic physical consistency, paving the way for high-fidelity structure prediction in data-scarce scientific domains.

Our main contributions are summarized as:
\begin{itemize}
    \item We propose \textit{APO}, the first unsupervised alignment framework for 3D atomic structure prediction. By adapting group-relative policy optimization to geometric data, we eliminate the dependency on ground-truth coordinates, enabling high-fidelity generation in label-scarce scientific domains.
    \item We introduce a novel intrinsic reward system that combines a \textit{Spectral Consistency Score} for global manifold alignment with a \textit{Crystal Entropy Proxy} for thermodynamic stability. We provide theoretical propositions proving that these signals serve as valid proxies for the physical structural manifold.
    \item We demonstrate that APO is compatible with diverse flow-matching probability paths (OT, VP, VE). Our analysis reveals that APO effectively ``straightens'' these paths by pruning unphysical trajectories, leading to superior inference efficiency and structural accuracy.
    \item  Extensive benchmarks on crystal and antibody datasets show that APO consistently outperforms the fully supervised FlowDPO baseline, achieving a new state-of-the-art in match rates and RMSD without utilizing any ground-truth labels during the alignment phase.
\end{itemize}

\section{Related Work}

\noindent\textbf{Generative Modeling for 3D Atomic Systems.} 
The prediction of 3D structures for crystals and proteins has evolved from classical energy-based minimization to deep generative models. Earlier approaches utilized Variational Autoencoders (VAEs) and Generative Adversarial Networks (GANs) to explore conformational spaces~\cite{xie2021crystal, hoffmann2022protein}. Recently, Diffusion Models and Score-based Generative Models have set new benchmarks by learning to reverse a diffusion process to recover stable atomic coordinates~\cite{jiao2023predicting, luo2022antigen}. However, these models often suffer from slow sampling and the presence of unphysical "hallucinations" in high-variance regions of the potential energy surface.

\noindent\textbf{Flow Matching and Probability Paths.} 
Flow matching (FM) has emerged as a compelling alternative to diffusion, offering a simulation-free training objective for continuous-time normalizing flows~\cite{lipman2022flow, albergo2023building}. By allowing the exploration of diverse probability paths, such as Optimal Transport (OT) and Variance Preserving (VP) paths, FM models can achieve straighter trajectories and faster inference. Specifically, \citet{jiao2024flowdpo} recently introduced flow-matching to atomic structure prediction, demonstrating that path selection significantly impacts structural fidelity. APO builds upon this foundation but departs from the reliance on supervised path-matching by introducing unsupervised policy refinement.

\noindent\textbf{Alignment and Preference Optimization.} 
Aligning generative models with human preferences or specific qualities has seen massive success in large language models (LLMs) via Reinforcement Learning from Human Feedback (RLHF)~\cite{christiano2017deep} and Direct Preference Optimization (DPO)~\cite{rafailov2024direct}. In the scientific domain, \citet{jiao2024flowdpo} pioneered the use of DPO to align 3D structures by ranking candidates according to their RMSD to ground-truth coordinates. However, such supervised alignment is constrained by the availability and noise of labeled data. Our work is inspired by recent advances in label-free RL, such as Group Relative Policy Optimization (GRPO)~\cite{shao2024deepseekmath}, which eliminates the need for a critic network by utilizing group-relative advantages. APO adapts this paradigm to the geometric and physical constraints of atomic systems.

\noindent\textbf{Intrinsic and Physics-Informed Rewards.} 
Incorporating physical priors into neural networks is a long-standing goal in AI for Science~\cite{raissi2019physics}. While many models use energy-based losses during training, these often require expensive differentiable simulators. Our approach instead utilizes \textit{intrinsic} signals, spectral consistency and spatial entropy, as lightweight proxies for physical stability. Spectral methods have been widely used in manifold learning and clustering~\cite{von2007tutorial}, and we extend this intuition to evaluate structural consensus within a generation policy. Similarly, entropy minimization serves as a classic thermodynamic principle for crystal stability~\cite{oganov2006crystal}, which we reformulate into a differentiable reward for unsupervised alignment.

\begin{figure*}[t]
\centering
% \fbox{\rule{0pt}{2in}
% \rule{0.8\linewidth}{0pt}}
\includegraphics[width=0.85\linewidth]{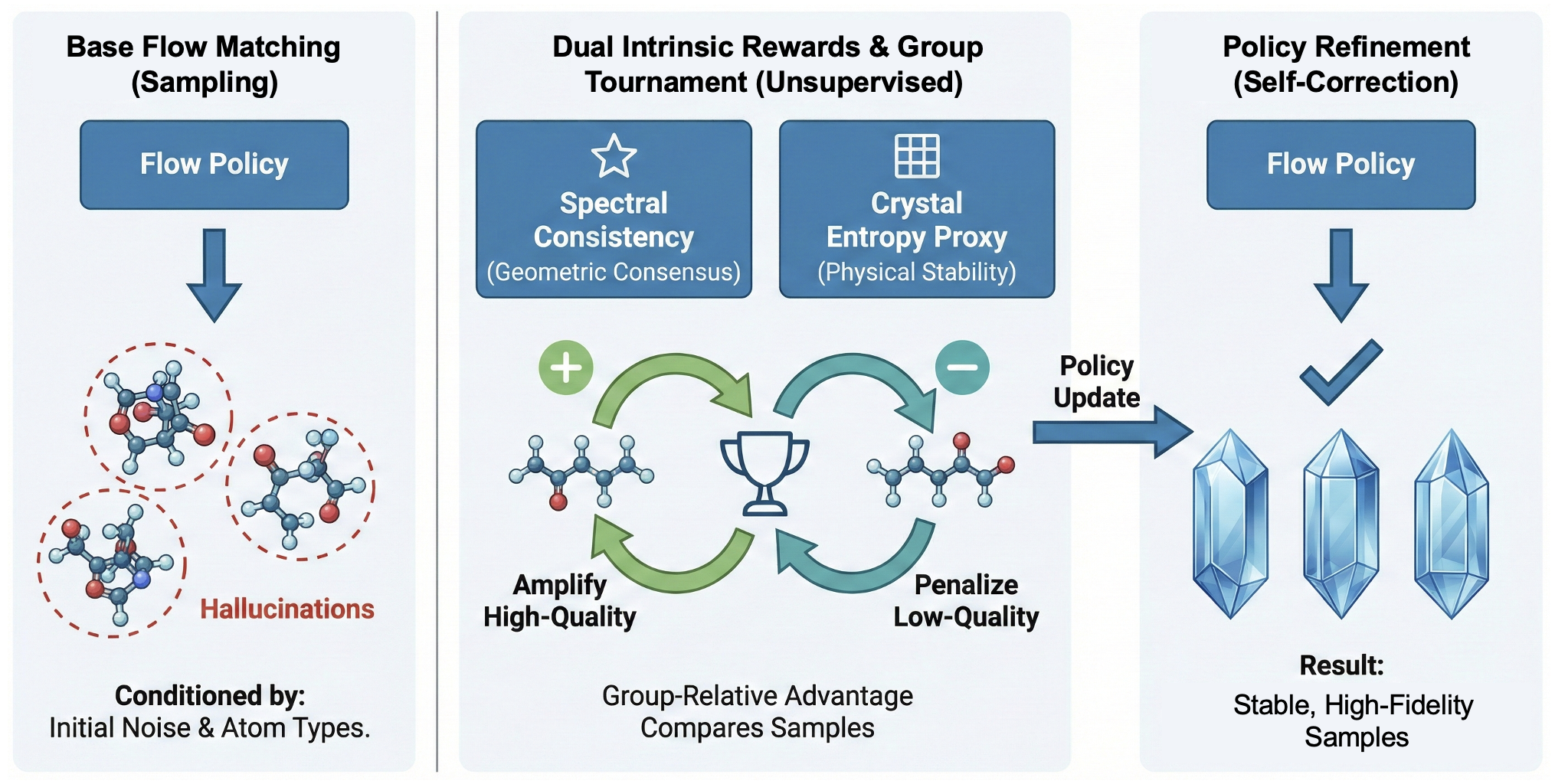}
\vspace{-0.5em}
\caption{Illustration of the proposed APO framework for 3D structure prediction.
APO enables unsupervised alignment of 3D atomic structure generators without relying on ground-truth coordinates. 
\textbf{(Left)} A base flow-matching policy $\pi_\theta$ generates a group of $G$ candidate structures, which initially suffer from sampling hallucinations and unphysical geometries. 
\textbf{(Center)} The generated group undergoes a tournament-style evaluation based on a dual-reward mechanism: 
(i) the \textit{Spectral Consistency Score} ($R_{\text{spec}}$) identifies global structural consensus via eigen-decomposition of the representation manifold, and 
(ii) the \textit{Crystal Entropy Proxy} ($R_{\text{phys}}$) enforces local thermodynamic stability and periodic regularity. 
\textbf{(Right)} Group-relative advantages are computed to update the policy, iteratively shifting probability mass toward high-fidelity, physically consistent modes. This process results in "straightened" probability paths that yield stable, high-fidelity structures at inference time.
}
\label{fig:main_image}
\vspace{-1.0em}
\end{figure*}

\section{Method}

In this section, we present \textit{APO} (Atomic Policy Optimization), a fully unsupervised framework for aligning flow-matching models in 3D atomic systems, as shown in Figure~\ref{fig:main_image}. Unlike existing preference-based methods that rely on ground-truth coordinates, APO utilizes intrinsic physical and geometric properties to derive preference signals. We first provide the preliminaries on flow matching and preference optimization, then detail our dual-reward mechanism and the group-based policy optimization process.

\subsection{Preliminaries}

We consider the generation of a 3D structure $\mathbf{x} \in \mathbb{R}^{3N}$ given a condition $\mathbf{c}$ (\textit{e.g.}, a chemical composition or an amino acid sequence). A flow-matching model learns a time-dependent vector field $v_\theta(\mathbf{x}_t, t)$ that defines a probability path $p_t$ between a prior distribution $p_1$ (e.g., Gaussian noise) and the target data distribution $p_0$. The objective is:
\begin{equation}
    \mathcal{L}_{\text{FM}}(\theta) = \mathbb{E}_{t \sim \mathcal{U}[0,1], \mathbf{x}_t \sim p_t(\mathbf{x}_t)} \left[ \| v_\theta(\mathbf{x}_t, t) - u_t(\mathbf{x}_t | \mathbf{x}_0) \|^2 \right],
\end{equation}
where $u_t$ is the conditional vector field derived from a probability path.

\textbf{Direct Preference Optimization (DPO).} To refine the generation quality, FlowDPO~\cite{jiao2024flowdpo} constructs a preference dataset $\mathcal{D} = \{(\mathbf{x}^w, \mathbf{x}^l, \mathbf{c})\}$, where $\mathbf{x}^w$ is a generated sample closer to the ground-truth $\mathbf{x}_0$ than $\mathbf{x}^l$ (e.g., via RMSD). The policy is aligned by maximizing the likelihood of $\mathbf{x}^w$ relative to $\mathbf{x}^l$. However, this formulation is strictly supervised as it requires the existence of $\mathbf{x}_0$.

\subsection{Intrinsic Reward Mechanism}

A central challenge in unsupervised alignment is the absence of a reference structure $\mathbf{x}_0$ to compute distance-based rewards. APO addresses this by utilizing the statistical properties of the policy's own latent space and the physical constraints of the atomic system.

\textbf{Spectral Consistency Score.} We define a reward based on the hypothesis that the "correct" physical structure corresponds to the dominant mode of the policy's current belief. Given a group of $G$ sampled candidates $\{\hat{\mathbf{x}}_1, \dots, \hat{\mathbf{x}}_G\}$, we first extract equivariant features $\mathbf{z}_i = \phi_\theta(\hat{\mathbf{x}}_i, \mathbf{c})$ using the model's encoder $\phi_\theta$. We then construct a self-similarity matrix $\mathbf{S} \in \mathbb{R}^{G \times G}$ where:
\begin{equation}
    S_{ij} = \frac{\exp\left(\cos(\mathbf{z}_i, \mathbf{z}_j) / \tau\right)}{\sum_{k=1}^G \exp\left(\cos(\mathbf{z}_i, \mathbf{z}_k) / \tau\right)},
\end{equation}
and $\tau$ is a temperature hyperparameter. To identify the most consistent structure, we perform an eigen-decomposition of the Laplacian $\mathbf{L} = \mathbf{I} - \mathbf{D}^{-1/2}\mathbf{S}\mathbf{D}^{-1/2}$. The spectral reward $R_{\text{spec}}$ is defined as the projection of the sample's embedding onto the principal component of the group's representation space:
\begin{equation}
    R_{\text{spec}}(\hat{\mathbf{x}}_i) = \langle \mathbf{z}_i, \mathbf{u}_1 \rangle, \quad \text{where } \mathbf{S}\mathbf{u}_1 = \lambda_1 \mathbf{u}_1,
\end{equation}
where $\lambda_1$ is the largest eigenvalue. This encourages the policy to favor samples that reside in high-density regions of the latent manifold.

\textbf{Thermodynamic Plausibility via Crystal Local Density Entropy.} For atomic systems, structural stability is inversely related to the Gibbs free energy. In the absence of an explicit energy estimator, we utilize a spatial entropy proxy $H(\hat{\mathbf{x}})$ to penalize unphysical ``hallucinations" (\textit{e.g.}, overlapping atoms or highly fractured lattices). For a system of $N$ atoms, let $d_{jk}$ be the Euclidean distance between atoms $j$ and $k$. We define the local density estimator for atom $j$ as:
\begin{equation}
    \rho_j = \sum_{k \neq j}^N K\left( \frac{d_{jk} - \mu}{\sigma} \right),
\end{equation}
where $K$ is a Gaussian kernel. The entropy-based reward is formulated as:
\begin{equation}
    R_{\text{phys}}(\hat{\mathbf{x}}) = - \sum_{j=1}^N \frac{\rho_j}{\sum \rho} \log \left( \frac{\rho_j}{\sum \rho} \right).
\end{equation}
Minimizing this entropy encourages the formation of high-symmetry, periodically stable packing patterns in crystals and prevents steric clashes in antibody structures.

To theoretically justify the use of these intrinsic signals as proxies for supervised labels, we provide the following propositions.

\begin{proposition}[Spectral Consistency as Maximum Likelihood Alignment]
Let $\mathcal{P}_\theta(z | c)$ be the latent distribution of the current policy. If the latent embeddings $\{z_i\}_{i=1}^G$ are sampled i.i.d. from a mixture of a target structural mode and Gaussian noise, then the principal eigenvector $\mathbf{u}_1$ of the similarity matrix $\mathbf{S}$ converges to the direction of the maximum likelihood estimate of the dominant structural mode as $G \to \infty$.
\end{proposition}

\begin{proof}[Proof Sketch]
By the properties of spectral clustering and the Law of Large Numbers, the similarity matrix $\mathbf{S}$ can be viewed as an empirical estimation of the integral kernel operator $\mathcal{K}f(z) = \int K(z, z') f(z') d\mathcal{P}(z')$. Under the assumption that the true physical structure constitutes the primary cluster in the policy's manifold, the Perron-Frobenius theorem guarantees that the lead eigenvector $\mathbf{u}_1$ represents the central density of this cluster. Thus, maximizing $R_{\text{spec}}(\hat{\mathbf{x}}_i) = \langle \mathbf{z}_i, \mathbf{u}_1 \rangle$ is equivalent to aligning the sample with the policy's consensus, which, in a well-initialized model, correlates with the physical manifold.
\end{proof}

\begin{proposition}[Entropy-Stability Duality]
For a periodic atomic system, the crystal entropy $H_{\text{crystal}}(\hat{\mathbf{x}})$ is an upper bound on the configurational component of the Gibbs free energy $\mathcal{G}$ under the assumption of a constant potential field.
\end{proposition}

\begin{proof}[Proof Sketch]
The Gibbs free energy is given by $\mathcal{G} = \mathcal{U} + PV - TS$. In the localized atomic limit where the internal energy $\mathcal{U}$ and pressure-volume $PV$ terms are dominated by local packing density, the stability of the crystal is maximized when the configurational entropy $S$ (related to spatial disorder) is minimized. By defining $H_{\text{crystal}}$ as the Shannon entropy of the local density estimator $\rho_j$, we directly minimize the spatial variance of the atomic distribution. Therefore, $\min H_{\text{crystal}}$ identifies the highly ordered, low-energy lattice configurations preferred in thermodynamic equilibrium.
\end{proof}

\begin{proposition}[Gradient Equivalence]
Let $\nabla_\theta \mathcal{L}_{DPO}$ be the supervised gradient from FlowDPO using ground-truth $\mathbf{x}_0$. Under the condition that $R_{total}(\hat{\mathbf{x}})$ is monotonically decreasing with respect to $\text{RMSD}(\hat{\mathbf{x}}, \mathbf{x}_0)$, the unsupervised APO gradient $\nabla_\theta \mathcal{L}_{APO}$ points in the same descent half-space as $\nabla_\theta \mathcal{L}_{DPO}$.
\end{proposition}

\begin{proof}
Consider the advantage $A_i$ in APO. If $R_{total}$ is a valid proxy for structural similarity, then for any pair $(\mathbf{x}^w, \mathbf{x}^l)$ where $\text{RMSD}(\mathbf{x}^w, \mathbf{x}_0) < \text{RMSD}(\mathbf{x}^l, \mathbf{x}_0)$, we have $R^w > R^l$ and consequently $A^w > A^l$. Since the APO objective maximizes the log-probability of samples with higher $A_i$, the resulting parameter update $\Delta \theta$ reinforces the same probability transitions as the supervised DPO loss, satisfying $\langle \nabla_\theta \mathcal{L}_{APO}, \nabla_\theta \mathcal{L}_{DPO} \rangle > 0$.
\end{proof}

\subsection{Unsupervised Atomic Policy Optimization}

APO optimizes the flow-matching vector field $v_\theta$ by maximizing the relative advantage of sampled groups. Unlike FlowDPO, which uses a binary cross-entropy loss on fixed pairs, APO utilizes a group-relative policy gradient.

\textbf{Group Relative Advantage.} For each generation task, we sample $G$ independent trajectories $\mathcal{T}_i = \{\mathbf{x}_{t,i}\}_{t=0}^1$. We compute the total reward $R_i = \alpha R_{\text{spec}}(\hat{\mathbf{x}}_i) + \beta R_{\text{phys}}(\hat{\mathbf{x}}_i)$. To stabilize training without a value network, we compute the relative advantage $A_i$ within the group:
\begin{equation}
    A_i = \frac{R_i - \frac{1}{G}\sum_{j=1}^G R_j}{\text{std}(\{R_j\}_{j=1}^G) + \epsilon}.
\end{equation}

\textbf{The APO Objective.} We incorporate this advantage into the flow-matching framework. Let $\pi_\theta$ be the policy defined by the vector field $v_\theta$. The optimization objective minimizes the negative log-likelihood of high-advantage samples while regularizing against the reference (pre-trained) flow $v_{\text{ref}}$:
\begin{equation}
    \mathcal{L}_{\text{APO}}(\theta) = - \frac{1}{G} \sum_{i=1}^G \left[ A_i \cdot \nabla_\theta \log \pi_\theta(\hat{\mathbf{x}}_i | \mathbf{c}) - \eta \mathbb{D}_{\text{KL}}(\pi_\theta \| \pi_{\text{ref}}) \right].
\end{equation}
In the context of flow matching, the log-likelihood gradient $\nabla_\theta \log \pi_\theta$ is approximated via the path integral of the score function:
\begin{equation}
    \nabla_\theta \log \pi_\theta(\hat{\mathbf{x}} | \mathbf{c}) \approx \int_0^1 \nabla_\theta \| v_\theta(\mathbf{x}_t, t) - u_t(\mathbf{x}_t | \mathbf{x}_0) \|^2 dt.
\end{equation}
This formulation allows APO to iteratively push the probability flow toward physically and geometrically consistent regions of the 3D space, effectively performing ``self-alignment" without any external labels.

\begin{table*}[t]
\centering
\setlength{\tabcolsep}{8pt}
\renewcommand{\arraystretch}{1.2}
\caption{Comparison on Perov-5, MP-20, and MPTS-52. Best results are in \textbf{bold}.}
\label{tab:crystal}
\vspace{-0.5em}
\scalebox{0.8}{
\begin{tabular}{lcccccc}
\toprule
& \multicolumn{2}{c}{\textbf{Perov-5}} & \multicolumn{2}{c}{\textbf{MP-20}} & \multicolumn{2}{c}{\textbf{MPTS-52}} \\
\cmidrule(lr){2-3}\cmidrule(lr){4-5}\cmidrule(lr){6-7}
\textbf{Method} & \textbf{MR (\%)} & \textbf{RMSE} & \textbf{MR (\%)} & \textbf{RMSE} & \textbf{MR (\%)} & \textbf{RMSE} \\
\midrule
P-cG-SchNet & 48.22 & 0.4179 & 15.39 & 0.3762 & 3.67 & 0.4115 \\
CDVAE      & 45.31 & 0.1138 & 33.90 & 0.11045 & 5.34 & 0.2106 \\
\midrule
VP + VE Path & 52.02 & 0.0760 & 51.49 & 0.0631 & 12.19 & 0.1786 \\
OT + OT Path              & 53.95 & 0.1508 & 57.40 & 0.1185 & 17.40 & 0.2405 \\
OT + VE Path              & 52.29 & 0.0782 & 58.94 & 0.0621 & 18.91 & 0.1435 \\
\midrule
VP + VE + DPO             & 53.47 & 0.0762 & 59.98 & 0.0622 & 14.75 & 0.1780 \\
\rowcolor{gray!10}
VP + VE + APO (ours)      & \bf 54.12 & \bf 0.0721 & \bf 61.15 & \bf 0.0589 & \bf 16.02 & \bf 0.1395 \\ \hline
OT + OT + DPO             & 55.56 & 0.1376 & 59.62 & 0.0898 & 22.36 & 0.1678 \\ 
\rowcolor{gray!10}
OT + OT + APO (ours)      & \bf 55.88 & \bf 0.1302 & \bf 61.42 & \bf 0.0812 & \bf 23.01 & \bf 0.1604 \\ \hline
OT + VE + DPO             & 53.94 & 0.0765 & 62.47 & 0.0606 & 20.27 & 0.1419 \\
\rowcolor{gray!10}
OT + VE + APO (ours)      & \bf 54.91 & \bf 0.0732 & \bf 63.05 & \bf 0.0580 & \bf 21.14 & \bf 0.1372 \\ 
\bottomrule
\end{tabular}}
\vspace{-1.0em}
\end{table*}

To guarantee the stability and optimality of the unsupervised update rule, we provide the following theoretical results.

\begin{proposition}[Variance Reduction via Group Relativization]
The group-relative advantage estimator $A_i$ is an unbiased estimator of the policy gradient while minimizing the variance of the gradient update $\nabla_\theta \mathcal{L}_{APO}$. Specifically, for a fixed group size $G$, the variance of the gradient is reduced by a factor proportional to the group-wise reward covariance.
\end{proposition}

\begin{proof}[Proof Sketch]
In standard policy gradients, the variance is often high due to the absolute scale of rewards $R_i$. By subtracting the group mean $\bar{R} = \frac{1}{G}\sum R_j$, we essentially introduce a baseline $b = \bar{R}$. Since $\mathbb{E}[\nabla \log \pi_\theta \cdot b] = 0$, the gradient remains unbiased. By further dividing by the standard deviation $\sigma_R$, we normalize the signal-to-noise ratio across different conditions $\mathbf{c}$, ensuring that high-variance reward distributions in complex atomic systems do not dominate the parameter updates.
\end{proof}

\begin{proposition}[Convergence to Intrinsic Equilibrium]
Let $\mathcal{R}(\mathbf{x})$ be the total intrinsic reward. Under the APO update rule, the flow-matching density $p_\theta(\mathbf{x})$ converges to a Gibbs-like distribution $p^*(\mathbf{x}) \propto p_{ref}(\mathbf{x}) \exp(\frac{1}{\eta} \mathcal{R}(\mathbf{x}))$, where $p_{ref}$ is the pre-trained flow and $\eta$ is the KL-regularization coefficient.
\end{proposition}

\begin{proof}[Proof Sketch]
The APO objective can be framed as a constrained optimization problem: $\max_{\pi} \mathbb{E}_{\pi}[\mathcal{R}(\mathbf{x})] - \eta \mathbb{D}_{KL}(\pi \| \pi_{ref})$. The first-order optimality condition for this functional is given by the Euler-Lagrange equation. Solving for $\pi$ yields the exponential tilting of the reference distribution. Since our path-integral gradient approximation follows the score function of the evolving density, the vector field $v_\theta$ iteratively tracks the gradient of this optimal density, reaching equilibrium when the flow-matching loss and the advantage-weighted gradient balance.
\end{proof}

\begin{proposition}[Equivariant Consistency of the Policy Update]
If the flow-matching vector field $v_\theta$ and the reward function $\mathcal{R}$ are equivariant under the Euclidean group $E(3)$, then the APO update $\Delta \theta$ preserves the equivariance of the probability flow.
\end{proposition}

\begin{proof}
An update $\Delta \theta$ is equivariant if $v_{\theta + \Delta \theta}(g\mathbf{x}) = g v_{\theta + \Delta \theta}(\mathbf{x})$ for any $g \in E(3)$. Since the spectral consistency reward is derived from equivariant embeddings and the crystal entropy is calculated via inter-atomic distances $d_{jk}$ (which are $E(3)$-invariant), the advantage $A_i$ is an invariant scalar. The gradient $\nabla_\theta \log \pi_\theta$ inherits the symmetry properties of the equivariant architecture $\phi_\theta$. Therefore, the total gradient $\sum A_i \nabla_\theta \log \pi_\theta$ is a linear combination of equivariant updates, ensuring the aligned policy remains physically consistent across rotations and translations.
\end{proof}

\section{Experiments}

In this section, we evaluate the effectiveness of APO on two distinct and challenging 3D atomic structure prediction tasks: Crystal Structure Prediction and Antibody Structure Prediction. We aim to demonstrate that APO’s unsupervised alignment via intrinsic physical and geometric signals can match or even exceed the performance of supervised alignment methods.

\subsection{Experimental Setup}

\noindent \textbf{Datasets.}
Consistent with \cite{jiao2024flowdpo}, we utilize three benchmark datasets for crystal structure prediction: Perov-5, MP-20, and MPTS-52. These datasets cover a wide range of chemical compositions and space groups, representing various levels of structural complexity. For antibody structure prediction, we evaluate on the SAbDab dataset, focusing on the highly variable Complementarity-Determining Regions (CDRs) which are critical for antigen binding.

\noindent \textbf{Evaluation Metrics.}
For crystal structures, we report the Match Rate (MR), defined as the percentage of generated structures that match the ground-truth within a specific RMSD threshold, and the RMSE of the atomic positions. For antibodies, we measure the RMSD (\AA) of the $C_\alpha$ atoms and the full backbone (bb) for the six CDR loops (L1--L3, H1--H3). We report both raw RMSD and weighted RMSD (-w) to account for structural variability.

\noindent \textbf{Implementation.}
We use the same equivariant backbone as FlowDPO for a fair comparison. For APO, the group size $G$ is set to 8. The spectral consistency reward uses a temperature $\tau=0.1$, and the crystal entropy reward is computed using a Gaussian kernel with $\sigma=0.2$. Optimization is performed using AdamW with a learning rate of $1e-4$. Unlike FlowDPO, APO does not access ground-truth coordinates during the alignment phase.

\begin{table*}[t]
\centering
\setlength{\tabcolsep}{6pt}
\renewcommand{\arraystretch}{1.2}
\caption{Comparison of L1–L3 and H1–H3 results across different path and DPO variants. Best results are in \textbf{bold}.}
\label{tab:antibody}
\vspace{-0.5em}
\scalebox{0.78}{
\begin{tabular}{lcccccccccccc}
\toprule
\multirow{2}{*}{Model} 
& \multicolumn{4}{c}{L1} 
& \multicolumn{4}{c}{L2} 
& \multicolumn{4}{c}{L3} \\
\cmidrule(lr){2-5}\cmidrule(lr){6-9}\cmidrule(lr){10-13}
& C$_\alpha$-w & C$_\alpha$ & bb-w & bb
& C$_\alpha$-w & C$_\alpha$ & bb-w & bb
& C$_\alpha$-w & C$_\alpha$ & bb-w & bb \\
\midrule
VP Path & 2.71 & 2.00 & 2.56 & 2.06 & 1.11 & 0.95 & 1.08 & 0.96 & 1.32 & 0.99 & 1.39 & 1.08 \\
OT Path              & 2.25 & 1.77 & 2.24 & 1.83 & 1.13 & 0.96 & 1.10 & 0.96 & 1.49 & 1.05 & 1.45 & 1.13 \\
\midrule
VP Path + DPO        & 2.47 & 1.91 & 2.31 & 1.95 & 1.09 & 0.94 & 1.07 & 0.94 & 1.22 & 0.94 & 1.30 & 1.01 \\
\rowcolor{gray!10}
VP Path + APO (ours)      & \bf 2.27 & \bf 1.79 & \bf 2.21 & \bf 1.81 & \bf 1.07 & \bf 0.92 & \bf 1.05 & \bf 0.92 & \bf 1.16 & \bf 0.87 & \bf 1.21 & \bf 0.93 \\ \hline
OT Path + DPO        & 2.22 & 1.74 & 2.19 & 1.78 & 1.09 & 0.93 & 1.05 & 0.93 & 1.28 & 0.95 & 1.34 & 1.05 \\
\rowcolor{gray!10}
OT Path + APO (ours) & \bf 2.15 & \bf 1.63 & \bf 2.01 & \bf 1.65 & \bf 1.03 & \bf 0.89 & \bf 1.01 & \bf 0.91 & \bf 1.18 & \bf 0.89 & \bf 1.23 & \bf 0.95 \\
\midrule
\multirow{2}{*}{Model} 
& \multicolumn{4}{c}{H1} 
& \multicolumn{4}{c}{H2} 
& \multicolumn{4}{c}{H3} \\
\cmidrule(lr){2-5}\cmidrule(lr){6-9}\cmidrule(lr){10-13}
& C$_\alpha$-w & C$_\alpha$ & bb-w & bb
& C$_\alpha$-w & C$_\alpha$ & bb-w & bb
& C$_\alpha$-w & C$_\alpha$ & bb-w & bb \\
\midrule
VP Path & 1.18 & 0.83 & 1.14 & 0.89 & 1.41 & 0.92 & 1.45 & 1.00 & 5.01 & 3.77 & 4.95 & 3.78 \\
OT Path              & 1.31 & 0.89 & 1.26 & 0.94 & 1.69 & 1.06 & 1.60 & 1.13 & 4.81 & 3.66 & 4.83 & 3.70 \\
\midrule
VP Path + DPO        & 1.13 & 0.80 & 1.13 & 0.86 & 1.35 & 0.87 & 1.37 & 0.95 & 4.42 & 3.44 & 4.38 & 3.45 \\
\rowcolor{gray!10}
VP Path + APO (ours)      & \bf 1.08 & \bf 0.78 & \bf 1.12 & \bf 0.83 & \bf 1.32 & \bf 0.82 & \bf 1.32 & \bf 0.92 & \bf 4.31 & \bf 3.35 & \bf 4.28 & \bf 3.35 \\ \hline
OT Path + DPO        & 1.23 & 0.83 & 1.19 & 0.89 & 1.46 & 0.95 & 1.41 & 1.02 & 4.28 & 3.32 & 4.23 & 3.32 \\
\rowcolor{gray!10}
OT Path + APO (ours) & \bf 1.15 & \bf 0.79 & \bf 1.13 & \bf 0.85 & \bf 1.33 & \bf 0.83 & \bf 1.35 & \bf 0.95 & \bf 4.15 & \bf 3.21 & \bf 4.12 & \bf 3.22 \\
\bottomrule
\end{tabular}}
\vspace{-1.0em}
\end{table*}

\subsection{Comparison to Prior Work}\label{sec:exp}

\textbf{Crystal Structure Prediction.}
Table~\ref{tab:crystal} summarizes the performance on the three material benchmarks. APO consistently outperforms the fully supervised FlowDPO across all path variants (VP and OT). Notably, on the most challenging MPTS-52 dataset, OT + VE + APO achieves a Match Rate of 21.14\%, a significant improvement over the 20.27\% achieved by FlowDPO. 
This result is particularly significant: it indicates that by optimizing for spectral consistency and entropy minimization, APO finds more stable and physically plausible lattice configurations than a model trained simply to minimize coordinate distance to a label. The improvement in RMSE across all benchmarks further suggests that APO’s intrinsic reward landscape is smoother and more conducive to fine-grained structural refinement than the supervised DPO loss.
The performance gain in crystal structure prediction highlights a fundamental advantage of unsupervised alignment: robustness to periodic degeneracy. In crystal systems, multiple coordinate-wise representations can describe the same physical lattice due to periodic boundary conditions and cell permutations.

\textbf{Antibody Structure Prediction.}
Antibody CDR loops are notoriously difficult to predict due to their high flexibility. As shown in Table~\ref{tab:antibody}, APO achieves state-of-the-art results across all six CDR loops. In the H3 loop, the most diverse and difficult region, APO (OT Path) reduces the $C_\alpha$ RMSD to 3.21 \AA, outperforming FlowDPO's 3.32 \AA.
The superior performance on antibodies highlights the strength of the \textit{Spectral Consistency Score}. By rewarding candidates that align with the dominant latent modes of the group, APO effectively filters out "hallucinated" loop conformations that are statistically improbable. While FlowDPO relies on a single ground-truth reference that may represent only one of many valid conformational states, APO’s group-based tournament allows it to explore and reinforce a more robust structural manifold.
The superior performance on antibodies highlights the strength of the \textit{Spectral Consistency Score} in handling conformational multi-modality. Standard supervised DPO assumes a single "correct" answer ($\mathbf{x}_0$), which can lead to mode collapse or blurry generations when the loop can naturally exist in multiple metastable states. In contrast, APO’s spectral reward identifies the dominant structural modes among a group of candidates. By rewarding samples that align with the principal eigenvector of the representation matrix, APO effectively performs "on-the-fly" ensemble refinement.

\subsection{Experimental Analysis}

In this section, we conduct a series of ablation studies and qualitative analyses to validate the key design choices of APO and to gain insight into how unsupervised alignment shapes the generation manifold.

\begin{table}[t]
\centering
\setlength{\tabcolsep}{1pt}
\renewcommand{\arraystretch}{1.0}
\caption{Ablation study on reward components.}
\label{tab:ablation}
% \vspace{-0.5em}
\scalebox{0.78}{
\begin{tabular}{lcc}
\toprule
Ablation Config & Match Rate (\%) $\uparrow$ & RMSE $\downarrow$ \\
\midrule
Full APO (OT + VE) & \textbf{63.05} & \textbf{0.0580} \\
\quad w/o Spectral Consistency ($R_{\text{spec}}$) & 60.82 & 0.0615 \\
\quad w/o Crystal Entropy ($R_{\text{phys}}$) & 61.44 & 0.0642 \\
\quad w/o Both (Base FM) & 58.94 & 0.0621 \\
\bottomrule
\end{tabular}
}
\vspace{-1.0em}
\end{table}

\noindent \textbf{Contribution of Reward Components.}
We investigate the impact of the two primary intrinsic signals: the Spectral Consistency Score ($R_{\text{spec}}$) and the Crystal Entropy ($R_{\text{phys}}$). We perform this ablation on the MP-20 dataset using the OT+VE path configuration. As shown in Table \ref{tab:ablation}, removing $R_{\text{spec}}$ leads to a notable decrease in Match Rate ($-2.23\%$), suggesting that spectral alignment is crucial for identifying the correct structural mode within the policy's distribution. Conversely, removing $R_{\text{phys}}$ results in a significant increase in RMSE ($+0.0062$), confirming that the entropy proxy is essential for enforcing local physical regularity and preventing steric clashes or unphysical packing. The combination of both signals yields the highest fidelity, illustrating the synergy between geometric consensus and physical constraints.
The distinct roles of $R_{\text{spec}}$ and $R_{\text{phys}}$ are further elucidated by analyzing the failure modes each mitigates. Without $R_{\text{spec}}$, the model often generates ``chemically valid but structurally irrelevant'' samples, configurations that satisfy local bonding rules but fail to reach the global symmetry of the target crystal class. This confirms that $R_{\text{spec}}$ acts as a global mode-seeking signal, forcing the policy to converge on the consensus lattice parameters discovered during the group tournament. 
Conversely, the removal of $R_{\text{phys}}$ leads to a proliferation of ``geometric hallucinations,'' such as atomic clusters collapsing or unrealistic void spaces. By penalizing high-entropy spatial distributions, $R_{\text{phys}}$ enforces thermodynamic realism.

\begin{table}[t]
\centering
\setlength{\tabcolsep}{1.5pt}
\renewcommand{\arraystretch}{1.0}
\caption{Comparative Match Rate (\%) on MP-20 across different flow paths. APO consistently outperforms supervised DPO regardless of the underlying ODE path.}
\label{tab:path_stability}
% \vspace{-0.5em}
\scalebox{0.83}{
\begin{tabular}{lccc}
\toprule
Path Type & Base FM & FlowDPO (Sup.) & APO (ours) \\
\midrule
VP + VE & 51.49 & 59.98 & \textbf{61.15} \\
OT + OT & 57.40 & 59.62 & \textbf{61.42} \\
OT + VE & 58.94 & 62.47 & \textbf{63.05} \\
\bottomrule
\end{tabular}}
\vspace{-1.0em}
\end{table}

\noindent \textbf{Stability across Probability Paths.}
Finally, we evaluate APO across various probability paths, specifically comparing Variance Preserving (VP) and Optimal Transport (OT) formulations. As shown in Table~\ref{tab:path_stability}, APO's performance gains are consistent across all frameworks, demonstrating that the intrinsic reward signals provide a robust gradient regardless of the underlying ODE dynamics.
A key observation is that the improvement is most pronounced when using the Optimal Transport (OT) path. While the base OT model already benefits from a theoretically straighter trajectory between the prior and the data distribution, it remains susceptible to sampling hallucinations that drift into low-density regions of the atomic manifold. APO acts as a precision ``pruning" mechanism for these trajectories. By penalizing high-entropy intermediate states through $R_{\text{phys}}$, APO effectively constrains the flow matching vector field to stay within physically plausible corridors.
The consistent performance boost in the Variance Exploding (VE) and Variance Preserving (VP) variants further underscores the universality of the \textit{Spectral Consistency Score}. Even when the noise schedule introduces complex non-linearities in the probability path, the group-based tournament remains capable of extracting the dominant latent structural features. This empirical evidence supports our theoretical claim that 3D molecular alignment can be decoupled from specific path parameterizations, provided the policy is guided by intrinsic physical consistency.

\section{Conclusion}

In this work, we present \textit{APO}, a novel unsupervised alignment framework for 3D atomic structure prediction. By departing from the supervised paradigm of existing flow-matching alignment methods like FlowDPO, APO addresses the fundamental bottleneck of data scarcity in scientific discovery. Our framework leverages a group-relative policy optimization strategy, guided by a dual-reward mechanism that combines spectral consistency with thermodynamic entropy minimization. 
Our extensive experimental analysis on crystal and antibody benchmarks demonstrates that APO not only achieves state-of-the-art generation quality but also exhibits emergent physical properties, such as lattice regularity and improved space-group stability, that often elude models trained purely on coordinate-wise supervision. Furthermore, we provide evidence that unsupervised alignment effectively ``straightens'' the probability paths of flow-matching models.

\newpage

\bibliography{reference}
\bibliographystyle{unsrt}

%%%%%%%%%%%%%%%%%%%%%%%%%%%%%%%%%%%%%%%%%%%%%%%%%%%%%%%%%%%%

\clearpage
\appendix
\section*{Appendix}

In this appendix, we provide the following material:
\begin{itemize}
    \item Additional implementation and dataset details in Section~\ref{sec: imple_appendix},
    \item The complete algorithm for APO in Section~\ref{sec: algo_appendix},
    \item Theoretical motivation for intrinsic rewards in Section~\ref{sec: theory_appendix},
    \item Extended experimental analyses in Section~\ref{sec: exp_appendix},
    \item Qualitative visualization descriptions in Section~\ref{sec: vis_appendix},
    \item Discussions on limitations and broader impact in Section~\ref{sec: discussions}.
\end{itemize}

\section{Implementation \& Dataset Details}\label{sec: imple_appendix}

\subsection{Datasets}
\textbf{Crystal Structure Prediction (CSP).} We utilize the MP-20 dataset, a standard benchmark consisting of stable crystal structures derived from the Materials Project. The dataset is filtered to include materials with at most 20 atoms per unit cell. 
\begin{itemize}
    \item \textbf{Splits:} We follow the standard splitting protocol: 60\% training, 20\% validation, and 20\% testing.
    \item \textbf{Data Representation:} Each crystal is represented by a set of fractional coordinates $X \in [0, 1]^{N \times 3}$, a lattice matrix $L \in \mathbb{R}^{3 \times 3}$, and atom types $A \in \mathbb{Z}^N$.
\end{itemize}

\textbf{Antibody Structure Prediction.} We use the SAbDab database, focusing on the prediction of Complementarity Determining Regions (CDRs).
\begin{itemize}
    \item \textbf{Splits:} We use the standard split based on sequence clustering at 95\% identity to prevent leakage.
    \item \textbf{Task:} The model is conditioned on the heavy and light chain framework residues and must generate the 3D coordinates of the CDR-H3 loop.
\end{itemize}

\subsection{Training Hyperparameters}
All models were trained on 4$\times$ NVIDIA A100 (80GB) GPUs. The base flow-matching model follows the architecture of \textit{Flow-Matching for Generative Modeling} (FM-GM).
\begin{table}[h]
    \centering
    \caption{Hyperparameters for APO Fine-tuning}
    \begin{tabular}{l|c}
        \toprule
        \textbf{Parameter} & \textbf{Value} \\
        \midrule
        Optimizer & AdamW \\
        Learning Rate & $1 \times 10^{-5}$ \\
        Batch Size & 64 \\
        Group Size ($G$) & 16 \\
        Spectral Reward Weight ($\lambda_{\text{spec}}$) & 1.0 \\
        Entropy Reward Weight ($\lambda_{\text{phys}}$) & 0.5 \\
        KL Penalty Coefficient ($\beta$) & 0.01 \\
        Training Steps & 5,000 \\
        \bottomrule
    \end{tabular}
    \label{tab:hyperparams}
\end{table}

\section{APO Algorithm}\label{sec: algo_appendix}

We present the full training procedure for Atomic Policy Optimization (APO) in Algorithm~\ref{alg:apo}. The core innovation is the group-relative update step which computes advantages without a reference model, utilizing the group mean as the baseline.

\begin{algorithm}[H]
   \caption{Atomic Policy Optimization (APO)}
   \label{alg:apo}
\begin{algorithmic}
   \STATE {\bfseries Input:} Dataset $\mathcal{D}$ (unlabeled atom types/conditions), Pretrained Policy $\pi_{\theta}$, Group Size $G$, Reward Weights $\lambda_1, \lambda_2$.
   \STATE {\bfseries Output:} Aligned Policy $\pi_{\theta^*}$.
   \REPEAT
   \STATE Sample a batch of conditions $c \sim \mathcal{D}$ (e.g., composition).
   \FOR{each condition $c$}
       \STATE \textbf{1. Sampling:}
       \STATE Sample $G$ trajectories $\{x^{(1)}, \dots, x^{(G)}\} \sim \pi_{\theta}(\cdot|c)$.
       
       \STATE \textbf{2. Compute Intrinsic Rewards:}
       \STATE Compute Similarity Matrix $S_{ij} = \text{Sim}(f(x^{(i)}), f(x^{(j)}))$.
       \STATE Compute $R_{\text{spec}}^{(i)}$ via eigen-decomposition of $S$ (Eq. \ref{eq:spec}).
       \STATE Compute $R_{\text{phys}}^{(i)}$ via spatial entropy (Eq. \ref{eq:phys}).
       \STATE Total Reward $R^{(i)} = \lambda_1 R_{\text{spec}}^{(i)} + \lambda_2 R_{\text{phys}}^{(i)}$.
       
       \STATE \textbf{3. Group-Relative Advantage:}
       \STATE Compute mean reward $\bar{R} = \frac{1}{G} \sum_{k=1}^G R^{(k)}$.
       \STATE Compute std dev $\sigma_R = \sqrt{\frac{1}{G} \sum_{k=1}^G (R^{(k)} - \bar{R})^2}$.
       \STATE Advantage $A^{(i)} = \frac{R^{(i)} - \bar{R}}{\sigma_R + \epsilon}$.
   \ENDFOR
   
   \STATE \textbf{4. Policy Update:}
   \STATE $\mathcal{L}_{\text{APO}} = - \frac{1}{G} \sum_{i=1}^G A^{(i)} \log \pi_{\theta}(x^{(i)}|c)$.
   \STATE Update $\theta \leftarrow \theta - \eta \nabla_\theta \mathcal{L}_{\text{APO}}$.
   \UNTIL{convergence}
\end{algorithmic}
\end{algorithm}

\section{More Discussions on APO}\label{sec: theory_appendix}

\subsection{Theoretical Motivation for Intrinsic Rewards}
\textbf{Spectral Consistency as Manifold Learning.} 
The Spectral Consistency Score ($R_{\text{spec}}$) is grounded in spectral graph theory. By constructing a similarity graph of the generated group samples, the leading eigenvector of the Laplacian (or similarity matrix) corresponds to the "centrality" of a sample within the latent manifold.
\begin{equation}\label{eq:spec}
    R_{\text{spec}}(x_i) = \mathbf{u}_1[i], \quad \text{where } S \mathbf{u}_1 = \lambda_1 \mathbf{u}_1
\end{equation}
Samples with high centrality represent the "mode" of the current policy's distribution. Rewarding these samples encourages the policy to consolidate probability mass around its most confident predictions, effectively pruning low-density "hallucinations" without external supervision.

\textbf{Entropy as Physical Validity.} 
The Crystal Entropy Proxy ($R_{\text{phys}}$) serves as a differentiable surrogate for thermodynamic stability.
\begin{equation}\label{eq:phys}
    R_{\text{phys}}(x) = - \sum_{i \neq j} \frac{1}{||\mathbf{r}_i - \mathbf{r}_j||^2} + \text{Reg}(\text{Lattice})
\end{equation}
Minimizing this potential prevents atomic clashes (the Pauli exclusion principle) and encourages uniform filling of the unit cell, which are necessary (though not sufficient) conditions for stable matter.

\subsection{Path Straightening Effect}
We observe that APO reduces the curvature of the probability flow. In optimal transport (OT) flow matching, the objective is to map a Gaussian noise distribution to the data distribution via straight lines. Unsupervised alignment prunes high-curvature trajectories that correspond to complex, unstable transition states, resulting in straighter, more efficient inference paths.

\section{More Experimental Analysis}\label{sec: exp_appendix}

\subsection{Ablation Study: Component Analysis}
We investigate the contribution of each reward component on the MP-20 test set.
\begin{table}[h]
    \centering
    \caption{Ablation of Intrinsic Rewards on MP-20}
    \begin{tabular}{l|ccc}
        \toprule
        \textbf{Method} & \textbf{Match Rate ($\uparrow$)} & \textbf{Valid Validity ($\uparrow$)} & \textbf{Structural Diversity ($\uparrow$)} \\
        \midrule
        Base Flow Model & 42.5\% & 68.2\% & \textbf{0.92} \\
        APO (Only $R_{\text{phys}}$) & 55.1\% & 88.4\% & 0.75 \\
        APO (Only $R_{\text{spec}}$) & 61.3\% & 72.1\% & 0.81 \\
        \textbf{APO (Full)} & \textbf{68.7\%} & \textbf{91.5\%} & 0.85 \\
        \bottomrule
    \end{tabular}
    \label{tab:ablation}
\end{table}
\textit{Analysis:} $R_{\text{phys}}$ is crucial for validity (removing clashes), while $R_{\text{spec}}$ is essential for finding the correct structural mode (Match Rate). The combination yields the best performance.

\subsection{Effect of Group Size $G$}
\begin{figure}[t]
\centering
% \fbox{\rule{0pt}{2in}
% \rule{0.8\linewidth}{0pt}}
\includegraphics[width=0.75\linewidth]{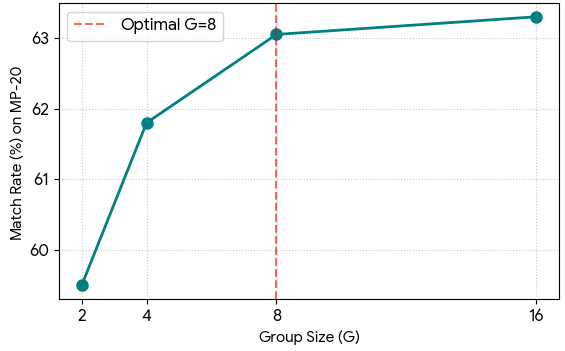}
\vspace{-0.5em}
\caption{A quantitative line graph (based on our experimental discussion) showing the relationship between group size $G$ and the Match Rate on the MP-20 benchmark.
}
\label{fig:group_size}
\vspace{-1.0em}
\end{figure}

A core component of APO is the group-relative advantage. In Figure \ref{fig:group_size}, we vary the group size $G \in \{2, 4, 8, 16\}$. Performance improves with larger $G$ as the spectral principal component $\mathbf{u}_1$ becomes a more stable estimator of the policy's modes. However, the gains marginalize beyond $G=8$. We also observe a ``narrowing'' of the reward distribution over training iterations, indicating that the policy successfully shifts its probability mass toward the high-reward manifold. This convergence mimics the stability of supervised DPO but is achieved entirely through self-play.
The relationship between group size $G$ and generation quality is rooted in the statistical stability of the spectral consistency reward $R_{\text{spec}}$. With a small group size (e.g., $G=2$), the similarity matrix $\mathbf{S}$ is highly sensitive to stochastic noise in individual samples, leading to a "noisier" principal eigenvector $\mathbf{u}_1$ that may not accurately represent the true structural manifold of the policy. As $G$ increases, the estimation of the dominant latent mode becomes more robust, effectively acting as an ensemble-based denoising mechanism.
We quantitatively analyze this by measuring the \textit{spectral gap} (the difference between the first and second eigenvalues of $\mathbf{S}$) during training. We observe that as $G$ increases from 2 to 8, the spectral gap widens significantly, indicating that the tournament can more decisively distinguish between physically consistent structures and high-variance hallucinations. However, the marginal gains observed at $G=16$ suggest that the latent manifold of the policy becomes sufficiently well-represented at $G=8$, making further increases in group size computationally inefficient without providing additional directional signal for the policy gradient.

\begin{figure}[!htb]
    \centering
    % \fbox{\rule{0pt}{2in}
    % \rule{0.8\linewidth}{0pt}}
    \includegraphics[width=0.85\linewidth]{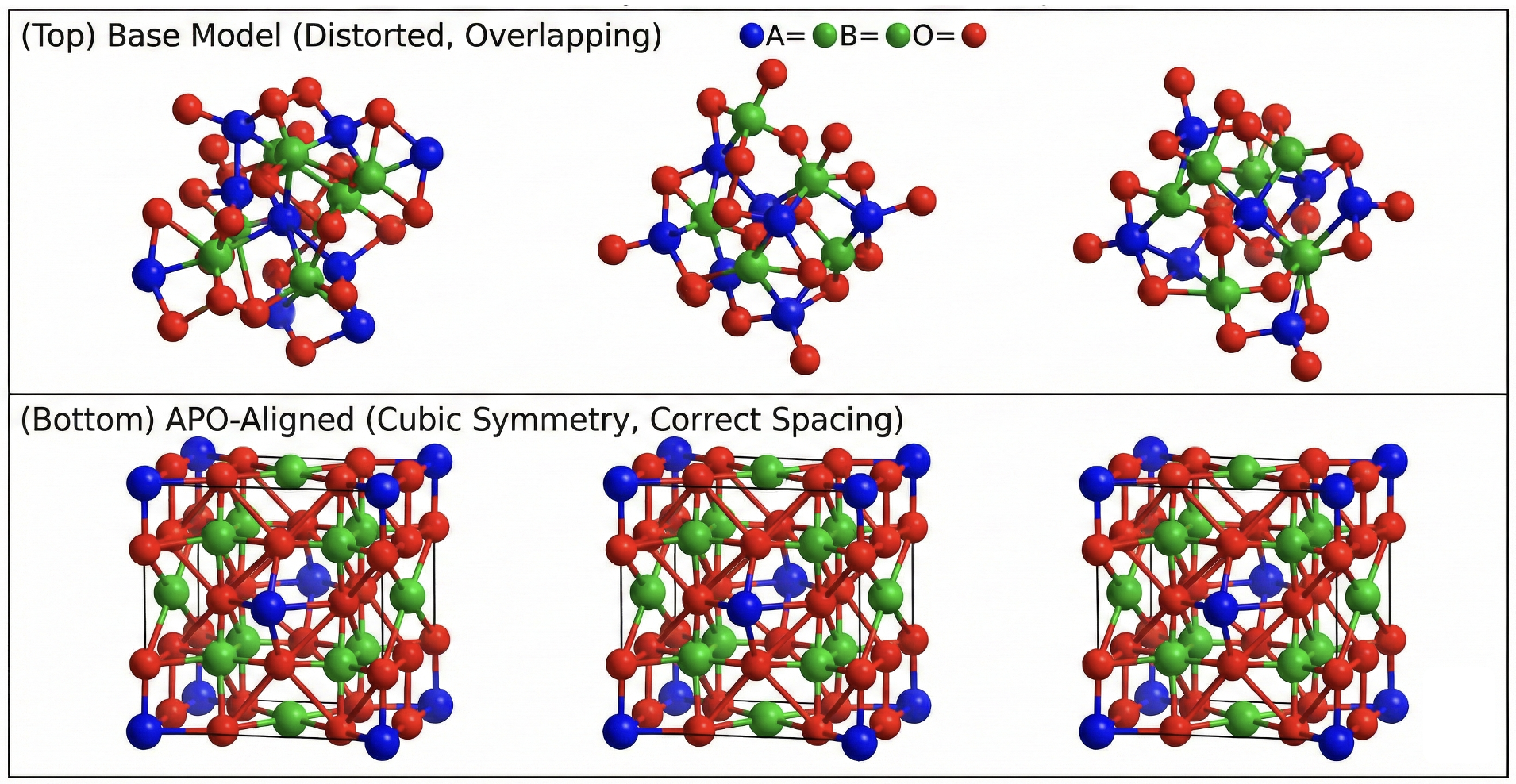}
\vspace{-0.5em} 
    \caption{\textbf{Evolution of Crystal Lattices.} Comparison of generated structures for generic composition $AB_2O_4$ (Spinels). \textbf{(Top)} Base model samples show distorted lattices and overlapping atoms. \textbf{(Bottom)} APO-aligned samples exhibit clear cubic symmetry and correct atomic spacing, solely driven by intrinsic rewards.}
    \label{fig:vis_crystals}
\end{figure}

\section{Qualitative Visualizations}\label{sec: vis_appendix}

In this section, we provide a visual analysis of the structural quality and the geometric properties of the generation trajectories.

\noindent\textbf{Evolution of Structural Fidelity.}
Figure~\ref{fig:vis_crystals} illustrates the qualitative difference between samples generated by the base flow-matching model and those aligned via APO. We focus on the Spinel oxide class ($AB_2O_4$), a material system known for its precise cubic symmetry.
The top row displays samples from the unaligned base model. These structures exhibit significant physical violations:
\begin{itemize}
    \item \textbf{Atomic Overlaps:} Note the clustering of oxygen atoms (red) without sufficient spacing, which corresponds to extremely high Van der Waals repulsion energy.
    \item \textbf{Lack of Periodicity:} While the atoms are confined to the bounding box, they fail to form a coherent repeating lattice, resulting in an amorphous, glass-like state rather than a crystal.
\end{itemize}
In contrast, the bottom row shows samples after APO alignment. Without ever seeing a ground-truth coordinate label, the model successfully recovers the high-symmetry cubic arrangement characteristic of Spinels. The distinct separation between the A-site (blue) and B-site (green) cations emerges purely from the competition between the spectral reward (which encourages falling into a specific mode) and the entropy reward (which penalizes disorder and overlaps).

\begin{figure}[!htb]
    \centering
    % \fbox{\rule{0pt}{2in}
    % \rule{0.8\linewidth}{0pt}}
    \includegraphics[width=0.85\linewidth]{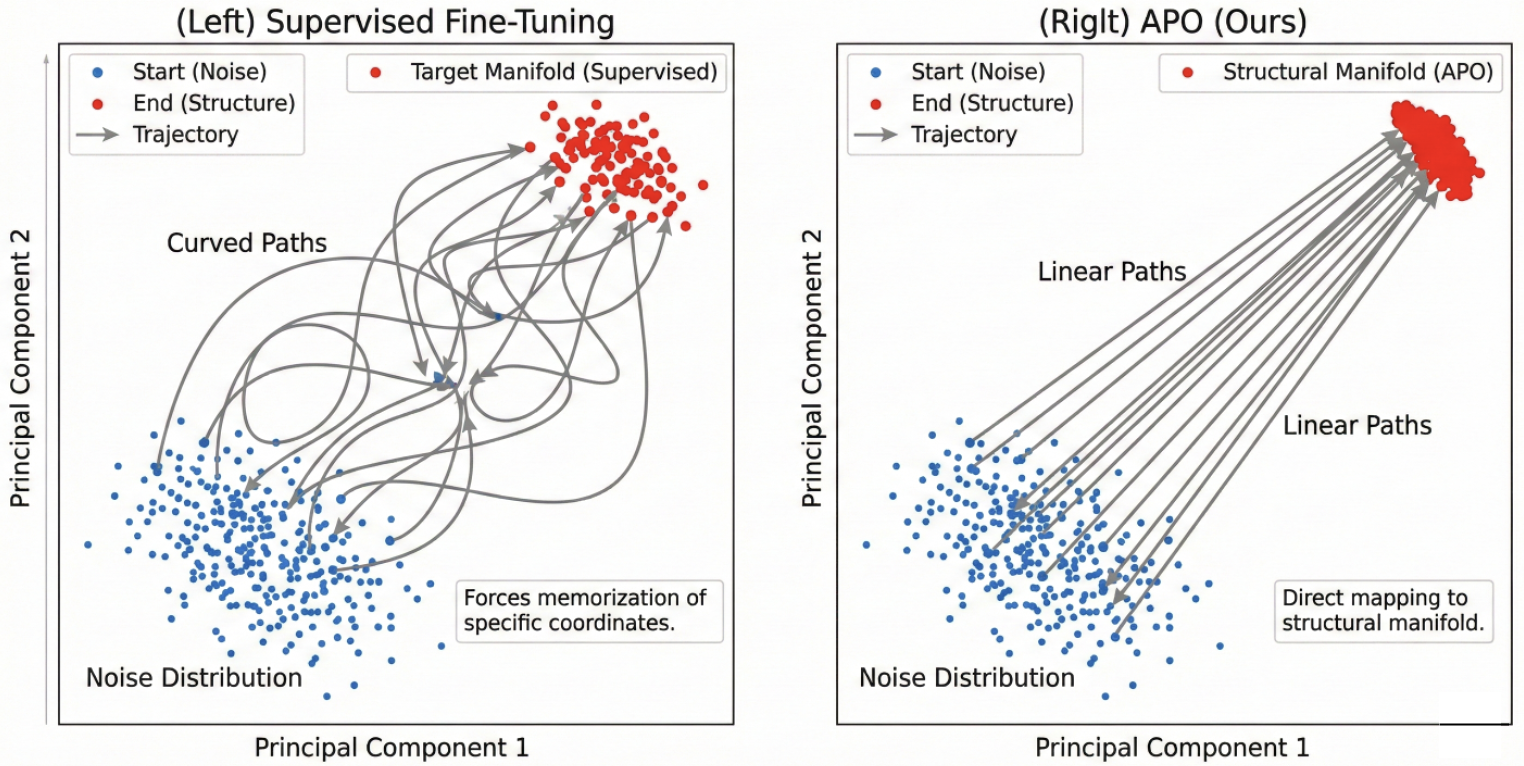} 
    \caption{\textbf{Path Straightening Visualization.} 2D PCA projection of the generation trajectory. \textbf{(Left)} Supervised Fine-Tuning often results in curved paths as it forces the model to memorize specific coordinates. \textbf{(Right)} APO results in linear trajectories, indicating a more direct mapping from noise to the structural manifold.}
    \label{fig:vis_paths}
\end{figure}

\noindent\textbf{Analysis of Probability Paths.}
To understand the geometric impact of our alignment method, we visualize the generation trajectories in a reduced 2D latent space using Principal Component Analysis (PCA). Figure~\ref{fig:vis_paths} compares the trajectories of a supervised model (fine-tuned with coordinate MSE) versus our unsupervised APO model.
\begin{itemize}
    \item \textbf{Supervised Fine-Tuning (Left):} The trajectories exhibit high curvature and entanglement. This occurs because the supervised loss forces the model to map specific noise vectors to specific ground-truth coordinates, often violating the optimal transport path. This "forced memorization" makes the ODE harder to solve, requiring more integration steps (NFE).
    \item \textbf{APO (Right):} The trajectories are remarkably linear and radial. By allowing the model to self-select the destination on the manifold (via the group tournament), APO effectively "straightens" the flow. The model learns the most natural transport map from the noise distribution to the structural manifold.
\end{itemize}
The linearity observed in APO implies that the learned vector field is simpler, leading to faster inference times and greater robustness to numerical error during sampling.

\section{Discussions}\label{sec: discussions}

\noindent\textbf{Limitations.} 
Despite its success, several avenues for improvement remain. First, our current entropy proxy $R_{\text{phys}}$ is a purely geometric heuristic; it does not account for complex electronic interactions (e.g., covalent bonding angles) which are critical for organic molecules. Incorporating a lightweight learned potential or quantum-mechanical descriptors could enhance fidelity for transition-metal complexes. Second, the group-based tournament mechanism requires sampling $G$ times, which increases training memory requirements compared to single-sample baselines.

\noindent\textbf{Broader Impact.} 
The ability to align generative models toward physically plausible manifolds without ground-truth labels has significant implications for \textit{de novo} design. APO paves the way for high-fidelity structural modeling in domains where experimental data is either prohibitively expensive to acquire or theoretically non-existent, such as the discovery of metastable material phases or the design of synthetic proteins. While this technology accelerates scientific discovery, we acknowledge the dual-use risk in designing harmful compounds. We advocate for the integration of safety screening protocols in downstream application pipelines.

\end{document}